\documentclass{article}
\usepackage{ijcai25}

\usepackage{times}
\usepackage{soul}
\usepackage{url}
\usepackage[utf8]{inputenc}
\usepackage[small]{caption}
\usepackage{booktabs}
\usepackage[switch]{lineno}

\usepackage{amsmath,amssymb,amsfonts}
\usepackage{algorithmic}
\usepackage{graphicx}
\usepackage{textcomp}
\usepackage{xcolor}
\usepackage{multirow}
\usepackage{algorithm}
\usepackage{amsthm,amsmath,amssymb}
\usepackage{mathrsfs}
\usepackage{hyperref}
\usepackage{booktabs} 
\usepackage{siunitx}  
\usepackage{subfigure}

\def\BibTeX{{\rm B\kern-.05em{\sc i\kern-.025em b}\kern-.08em
    T\kern-.1667em\lower.7ex\hbox{E}\kern-.125emX}}
    
\hypersetup{colorlinks = true, 
	    linkcolor = black, 
	    urlcolor = blue,
           citecolor = blue} 

\title{Stabilizing Quantization-Aware Training by Implicit-Regularization on Hessian Matrix}

\author{
Junbiao Pang$^1$
\and
Tianyang Cai$^1$\\\
\affiliations
$^1$Beijing University Of Technology\\
\emails
junbiao\_pang@bjut.edu.cn,
tianyang\_cai@163.com
}

\begin{document}

\maketitle

\begin{abstract}

Quantization-Aware Training (QAT) is one of the prevailing neural network compression solutions. However, its stability has been challenged for yielding deteriorating performances as the quantization error is inevitable. We find that the sharp landscape of loss, which leads to a dramatic performance drop, is an essential factor that causes instability. Theoretically, we have discovered that the perturbations in the feature would bring a flat local minima. However, simply adding perturbations into either weight or feature empirically deteriorates the performance of the Full Precision (FP) model. In this paper, we propose Feature-Perturbed Quantization (FPQ) to stochastically perturb the feature and employ the feature distillation method to the quantized model. Our method generalizes well to different network architectures and various QAT methods. Furthermore, we mathematically show that FPQ implicitly regularizes the Hessian norm, which calibrates the smoothness of a loss landscape. Extensive experiments demonstrate that our approach significantly outperforms the current State-Of-The-Art (SOTA) QAT methods and even the FP counterparts.

\end{abstract}

\section{Introduction}
\label{sec:intro}
Model compression has become an essential requirement for integrating deep models into edge computing devices. The prevalent methods in the domain of model compression include the search for optimal neural architectures~\cite{zoph-nas-arxiv-2016}, network pruning~\cite{han2-purn-arxiv-2015}, and the Deep Neural Network (DNN) quantization~\cite{li-brecq-arxiv-2021}~\cite{esser-lsq-arxiv-2019}. DNN quantization are categorized into two sub-classes: Post-Training Quantization (PTQ)~\cite{nagel-adaround-icml-2020},~\cite{li-brecq-arxiv-2021},~\cite{wei-qdrop-arxiv-2022},~\cite{li-ptqvit-tnnls-2023} and Quantization-Aware Training (QAT)~\cite{esser-lsq-arxiv-2019},~\cite{nagel-ooq-icml-2022}. PTQ adjusts the quantized model with a limited calibration dataset, bypassing the need for retraining. However, when dealing with low-bit widths, \textit{e.g.}, 2 or 4 bits, PTQ may face a significant drop in performance. Conversely, QAT re-trains or fine-tunes the neural network to maintain its accuracy especially for an extremely low-bit width model.

\begin{figure}[t!]
    \centering
    \includegraphics[width=1.0\linewidth]{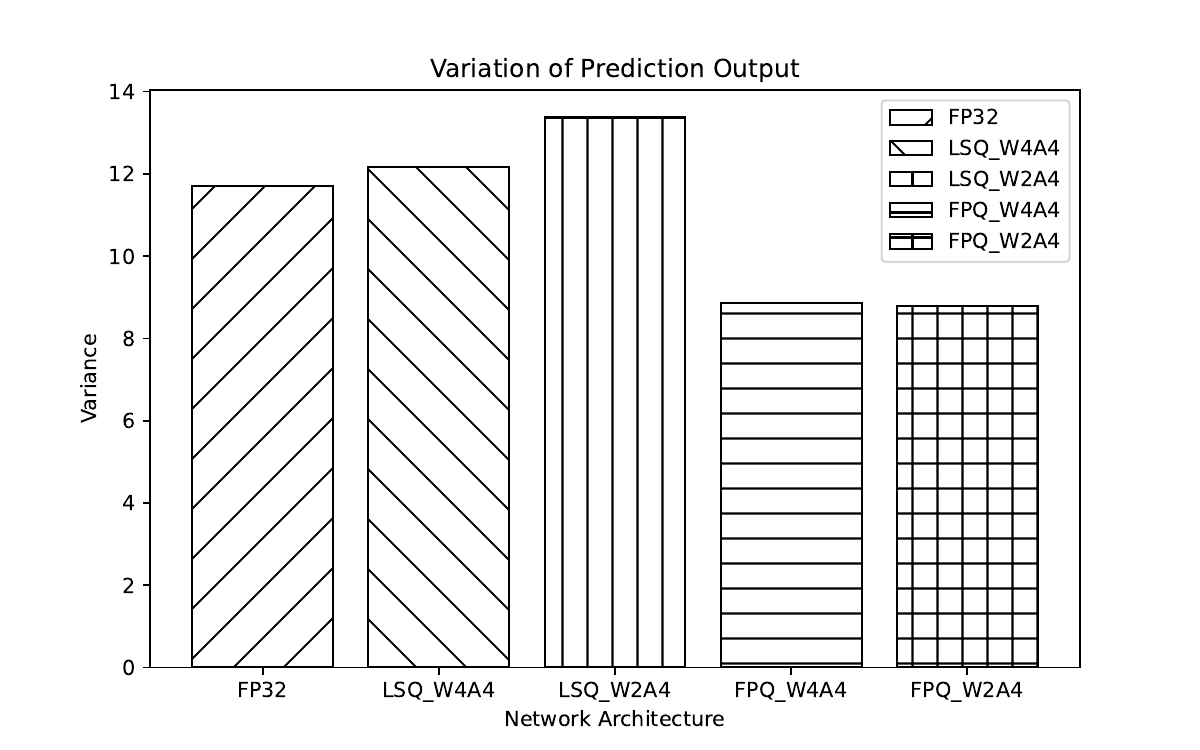}
    \caption{The output variances of FP32, LSQ W4A4, LSQ W2A4, FPQ W4A4, FPQ W2A4 models of the ResNet18 on CIFAR-10 dataset.}
    \label{fig:evidence}
\end{figure}

Despite QAT tries to preserve the performance of the low bit model by re-training (or fine-tuning) the weights, the stability of a QAT model
(\textit{i.e.}, the empirical performance is sensitive to the noise from the input for the quantizated model with the different bit widths) has been challenged. Importantly, the stability of a QAT model is considered to be the empirical way to calibrate the Lipschitz constant of a network~\cite{lin-dq-arxiv-2019}. That is, if the quantization error occurs, how stable is the output of a neural network? 

We have applied Gaussian-distribution perturbations to the original inputs to investigate the stability of the quantized DNN. Fig.~\ref{fig:evidence} illustrates that as the model parameters are optimized, the stability of FP and the SOTA QAT method (LSQ~\cite{esser-lsq-arxiv-2019}) deteriorate. Such instability makes QAT converge to the sharp loss landscape which leads to the significant performance drops when the model is quantized. 
This phenomenon discovers that the quantized model is highly sensitive to small perturbations from the input~\cite{lin-dq-arxiv-2019} or the quantization error either from the weights or the activations~\cite{nagel-datafreequantization-cvpr-2019}. For instance, BENN~\cite{zhu-benn-cvpr-2019} ensembled multiple binary models to alleviate this problem.
However, the inherent instability starts from the QAT training scheme and the ensambling approach is a compromise with the multiple quantized models for a prediction. 

The important sources of such instability come from three dimensions: 1) Straight-Through Estimator (STE)-based QAT has a bias gradient~\cite{shin-nipq-cvpr-2023}; 2) the quantization error would be propagated into the next layer of DNN~\cite{pang-qatcr-arxiv-2024}; and 3) the FP model itself is unstable due to the improper Lipsciz constant~\cite{qian-L2-arxiv-2018}. When these factors appear together for a quantized DNN, there is often a huge performance drop. Given the above factors as $\theta_1,\theta_2,\theta_3$ receptively, the corresponding loss is defined as $L(\theta_1, \theta_2, \theta_3)$. 
Let $f_{FP}^{(l)}$ and $f_{Q}^{(l)}$ denote the features at the $l$-th layer for the FP model and the quantized model, respectively. The quantization would introduce errors into the features as follows:
\begin{equation}
\label{eqt:feature_error}
f_{\text{Q}}^{(l)} = f_{\text{FP}}^{(l)} + \Delta f^{(l)}(\theta_1, \theta_2, \theta_3),
\end{equation}
where $ \Delta f^{(l)}(\theta_1, \theta_2, \theta_3)$ represents the perturbation in the feature caused by the factors $\theta_1, \theta_2$ and $\theta_3$ at the $l$-th layer. The perturbation $\Delta f^{(l)}(\theta_1, \theta_2, \theta_3)$ would propagate through all the layers of a neural network, leading to the variance in the final output as follows: 
\begin{equation}
\label{eqt:instability_form}
    L(\theta_{1}, \theta_{2}, \theta_{3}) =  L_{\text{original}} + \Delta L(\theta_{1}, \theta_{2}, \theta_{3}),
\end{equation}
where the $L_{\text{original}}$ is from the unperturbed model, and $\Delta L(\theta_{1}, \theta_{2}, \theta_{3})$ represents the loss caused by about three factors. The perturbation in loss $\Delta L(\theta_{1}, \theta_{2}, \theta_{3})$ is the errors between the quantized features and the FP features at each layer, which is controlled by the Lipsciz constants of different layers~\cite{pang-qatcr-arxiv-2024}.
In summary, QAT lacks the mechanism to ``absorb'' the perturbations (or errors\footnote{if it doesn't cause confusion, we will abuse use these terms.}) in each layer of the network caused by the above three aspects. 

As a result, even minor perturbations would lead to a substantial drop in performance, which is demonstrated as the poor generalization ability. 
As shown in Fig.~\ref{fig:evidence}, the instability leads a quantizated model to converge to a local sharp minima~\cite{deng-ausam-arxiv-2024}. 
Concretely, the sharp minima in the landscape illustrates that the network weight is almost only applicable to the current samples. 

Based on the above analysis, we propose Feature-Perturbed Quantization (FPQ) that leverages implicit-regularization on hessian matrix to smooth the loss landscape for overcoming the instability problem. Intuitively, the perturbation makes the optimization of the weights perform well on the nearby result rather than exactly the current one. Consequently, we control perturbations of features in each layer during the training stage by encouraging the model to converge to flat minima. 
Besides, we theoretically discover the advantage of Stochastic Feature Perturbations (SFP), and why SFP should be combined with a feature distillation technique. 
Our contributions are as follows:


\begin{itemize}
\item
Mathematically,  we show that the performance drop caused by quantization error is highly related to the norm of Hessian, which
is also mentioned empirically in~\cite{deng-ausam-arxiv-2024}. Furthermore, we discover that minimizing perturbations in the feature of each layer implicitly minimizes this term. Theoretically, it is shown that adding noise to features can explicitly optimize the flatness of the model and improve the generalization ability of the model.

\item 
We propose a approach to overcome the instability of QAT by leveraging implicit Hessian regularization. Instead of simply adding perturbations to the weights, we formulate the smoothness of landscape as the minimization of stochastic perturbation in feature, defined as the expected loss within the neighborhood of current neural network parameters. We theoretically discover the advantage of Stochastic Feature Perturbations (SFP), and why SFP should be combined with a feature distillation technique which leads to the improved performance. 


\item The proposed methods consistently improve QAT-based methods and match or improve the SOTA results on various network architectures on the CIFAR-10 and CIFAR-100. Besides, extensive experiments show that our methods outperform other QAT approaches.

\end{itemize}


\section{Related Work}
\subsection{Quantization-Aware Training}

As discussed in Section~\ref{sec:intro}, the errors of the first two dimensions are related to QAT, while the third one pertains to the process of pre-training FP models, which is not the focus of this paper.

In the first dimension, the instability is attributed to biased backpropagated gradients due to the round operation. For instance, Straight-Through Estimator (STE)~\cite{bengio-ste-arxiv-2013} utilized the expected probability of stochastic quantization as the gradient value for the backpropagation. 
The EWGS~\cite{lee-ewgs-cvpr-2021} adaptively adjusted the quantized gradients of the quantization error, thereby compensating for the biased gradient. The PSG~\cite{kim-psg-nips-2020} scaled the gradients based on the position of the weight, effectively offering a form of gradient compensation. The instability of gradients also leads to the oscillation problem during the QAT learning process. The DiffQ~\cite{defossez-diffq-arxiv-2021} identified that STE would lead to weight oscillations during training. The Overcoming Oscillation Quantization (OOQ)~\cite{nagel-ooq-icml-2022} addressed the oscillation issues by encouraging the latent weights to align closely with the center of the quantization bin. The Resilient Binary Neural Network (ReBNN)~\cite{xu-rebnn-aiii-2023} introduced a weighted reconstruction loss to formulate an adaptive training objective.

In the second dimension, DQ \cite{lin-dq-arxiv-2019} controled the error in feature transmission by managing the Lipschitz constant.  NIPQ~\cite{shin-nipq-cvpr-2023} employed pseudo-quantization noise to simulate the quantization process, reducing quantization errors. Similarly, some PTQ efforts aligned with the features of the FP model through layer-wise \cite{nagel-adaround-icml-2020} and block-wise~\cite{li-brecq-arxiv-2021},~\cite{wei-qdrop-arxiv-2022} reconstruction, thereby preventing the further propagation of errors within the network. To the best of our knowledge, we are the first to apply this idea to stabilize the QAT. 

Specially, for the quantization of Stable Diffusion (SD), the phenomenon of error accumulation is more pronounced due to the multi-step inference of the denoising process. BitsFusion~\cite{sui-itsfusion-arxiv-2024b} employs a mixed-precision approach to mitigate the accumulation of errors during the denoising process. Moreover, most of the work on SD quantization \cite{li-qdiffusion-cvpr-2023}, \cite{huang-tfmq-cvpr-2024}, \cite{he-ptqd-nips-2024}, \cite{so-temporal-nips-2024} is based on BRECQ, as the block-wise reconstruction in BRECQ can minimize error propagation as much as possible. However, this work differs from their focus. Their concern is addressing the accuracy issues caused by error accumulation, while we are concerned with the instability issues arising from error propagation.

\subsection{Adversarial Robustness}

In this paper, we argue that QAT should exhibit robustness against feature perturbations. This aim aligns with the topic of adversarial robustness, which focused on mitigating the vulnerability of neural networks to the perturbations from input \cite{szegedy-intriguing-arxiv-2013}, \textit{e.g.}, random smoothing \cite{lecuyer-certified-ieee-2019},~\cite{cohen-certified-pmlr-2019}, and adversarial training \cite{goodfellow-explaining-arxiv-2014},~\cite{mkadry-towards-iclr-2017}. Adversarial training, in particular, optimized the worst-case training loss and has been shown to not only improve robustness but also enhanced performance in tasks such as image classification \cite{xie-adversarial-cvpr-2020}. To the best of our knowledge, we are the first to apply this idea to stabilize QAT.

\section{Proposed Method}

\subsection{Notation and Background}

\textbf{Basic Notations.} In this paper, $\boldsymbol{x}$ represents a matrix (or tensor), a vector is denoted as  $\boldsymbol{x}$, $f(\boldsymbol{x};\boldsymbol{w}$) represents a FP model with the weight $\boldsymbol{w}$ and the input $\boldsymbol{x}$, $f(\boldsymbol{x};\boldsymbol{w}, \boldsymbol{s}, \boldsymbol{z})$ represents a quantized model with the parameter $\boldsymbol{w}$, quantization parameter $\boldsymbol{s}$, $\boldsymbol{z}$ and the input $\boldsymbol{x}$. We assume sample $\boldsymbol{x}$ is generated from the training set $\mathscr{D}_{t}$. 

\textbf{Quantization.} Step size $\boldsymbol{s}$ and zero point $\boldsymbol{z}$ serve as a bridge between floating-point and fixed-point representations. Given the input tensor $\boldsymbol{x}$\footnote{It could either be feature $\boldsymbol{x}$ or weight $\boldsymbol{w}$.}, the quantization operation is as follows:
\begin{equation}\label{eqt:quantization}
    \begin{aligned}
    \boldsymbol{x}_{int} &= clip\left (   \lfloor{\frac{\boldsymbol{x}}{\boldsymbol{s}}}\rceil + \boldsymbol{z},0,2^{q}-1 \right ),\\
    \hat{\boldsymbol{x}} &=\left ( \boldsymbol{x}_{int}-\boldsymbol{z}  \right ) \boldsymbol{s},
    \end{aligned}
\end{equation}
where $\lfloor{\cdot  }\rceil$ represents the rounding-to-nearest operator, $q$ is the predefined quantization bit-width, $\boldsymbol{s}$ denotes the scale between two subsequent quantization levels.  $\boldsymbol{z}$ stands for the zero-points. Both $\boldsymbol{s}$ and $\boldsymbol{z}$ are initialized by a calibration set $\mathscr{D}_{c}$ from the training dataset $ \mathscr{D}_{t}$, \textit{i.e.}, $\mathscr{D}_{c}\in \mathscr{D}_{t}$.
\begin{eqnarray}
\boldsymbol{s} =\frac{\boldsymbol{x}_{max} - \boldsymbol{x}_{min}}{2^{q} - 1},\label{eqt:scale_init} \\
\boldsymbol{z} =\lfloor{q_{max}-\frac{\boldsymbol{x}_{max}}{\boldsymbol{s}}}\rceil,\label{eqt:zero_init}
\end{eqnarray}
where $q_{max}$ is the maximum value of the quantized integer.

We follow the practice in LSQ~\cite{esser-lsq-arxiv-2019}, where $\boldsymbol{s}$ is a learnable parameter. Therefore, the loss function of a quantized model is given as follows:
\begin{equation}\label{eqt:Loss_example}
 \mathop{\arg\min}_{\boldsymbol{w},\boldsymbol{s}} \ \  \mathbb{E}_{\boldsymbol{x}\sim \mathscr{D}_{t}} [L\left (\boldsymbol{x}; \boldsymbol{w}, \boldsymbol{s} \right)],
\end{equation}
where $L\left (\cdot;\cdot\right)$ is the predefined loss function. 

Notice that, the zero-points $\boldsymbol{z}$ are initialized by \eqref{eqt:zero_init} with the calibration set $\mathscr{D}_{c}$, and are fixed throughout the entire QAT training process. 
During QAT training, the weights involved in the forward propagation are actually the quantized weights $\hat{\boldsymbol{w}}$, rather than the floating-point weights $\boldsymbol w$.

\subsection{Feature-Perturbation brings Implicit Regularization on Hessian Matrix}
\label{sec:FPQ}

\subsubsection{Motivation.}

The convergence property of oscillation near the target value is greatly beneficial to enhance the robustness of the network~\cite{chen-stabilizing-icml-2020}. When some perturbation $\boldsymbol{\delta}$ exists, the objective function can be approximated via Taylor expansion as follows:
\begin{equation}\label{eqt:talor_3}
\begin{split}
&\mathbb{E}[L(\boldsymbol{w} + \boldsymbol{\delta})]\\
\approx &\mathbb{E}[L(\boldsymbol{w}) + \boldsymbol{\delta} \cdot \nabla_{\boldsymbol{w}} L(\boldsymbol{w}) + \frac{1}{2} \boldsymbol{\delta}^T \cdot \nabla_{\boldsymbol{w}}^2 L(\boldsymbol{w}) \cdot \boldsymbol{\delta}] \\
\approx &L(\boldsymbol{w}) +\mathbb{E}[\boldsymbol{\delta}]\cdot \nabla_{\boldsymbol{w}} L(\boldsymbol{w})  + \frac{\mathbb{E}[\boldsymbol{\delta}^2]}{2} \operatorname{Tr}(\nabla_{\boldsymbol{w}}^2 L(\boldsymbol{w})) 
\end{split}
\end{equation}
where $\boldsymbol{w}$ is the weight of a DNN $f(\boldsymbol w)$. The third term of Eq.~\eqref{eqt:talor_3} involves the Hessian matrix's trace.~\cite{keskar-eigenvalue1-iclr-2017} and \cite{Wen-eigenvalue2-icais-2020} verified that the smaller the trace of the Hessian matrix is, the flatter the landscape of loss is. 
A flat loss surface generally aids the model in finding good local optima to enhance the model's generalization capability. 

\begin{figure}[t!]
    \centering
    \includegraphics[width=0.8\linewidth]{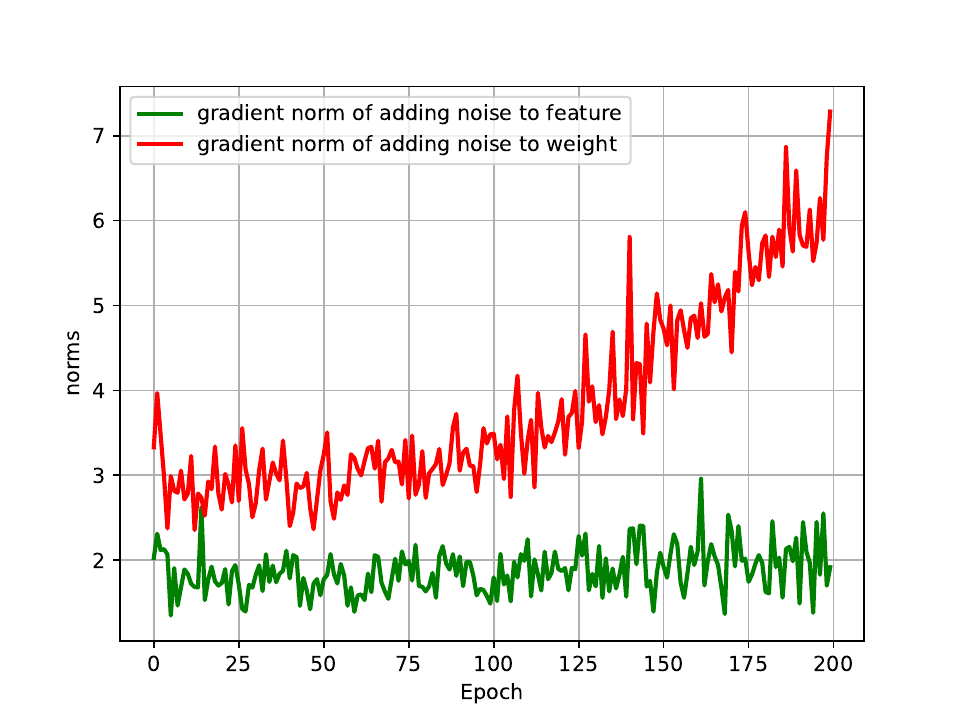}
    \caption{A comparison of the norm $ \|\nabla_{\boldsymbol{w}} L(\boldsymbol{w})\|$ trajectories of ResNet-18 on CIFIAR-10.}
    \label{fig:norm_cmp}
\end{figure}

We should make the second term in Eq.~\eqref{eqt:talor_3} be zero as follows:
\begin{equation}
\label{eqt:first_order_of_taylor}
    \mathbb{E}[\boldsymbol{\delta}]\cdot \nabla_{\boldsymbol{w}} L(\boldsymbol{w})=0.
\end{equation}
There are two conditions make the Eq.\eqref{eqt:first_order_of_taylor} be as zero as possible:
\begin{itemize}
    \item Condition 1: The $\nabla_{\boldsymbol{w}} L(\boldsymbol{w})$ should be close enough to zero.
    \item Condition 2: The expectation of perturbation $\mathbb{E}[\boldsymbol{\delta}]$ of every layer should be zero.
\end{itemize}
Intuitively, when a well-trained model $f(\boldsymbol{x},\boldsymbol{w})$ is used to fine-tune the QAT model, the $\nabla_{\boldsymbol{w}} L(\boldsymbol{w})$ would be nearly equal to zero. However, when a certain perturbation $\boldsymbol{\delta}$ is injected into either the weight or the feature of each layer, a natural question is: which scheme would incur a smaller norm of the gradient $\nabla_{\boldsymbol{w}} L(\boldsymbol{w})$ than the other?   

Therefore, we investigated the norm of the gradient $\nabla_{\boldsymbol{w}} L(\boldsymbol{w})$ when the perturbation $\boldsymbol{\delta}$ is  injected into either the weight or the feature of resnet-18 model on the CIFAR-10 dataset. Results in Fig.~\ref{fig:norm_cmp} show that the norm of the gradient norm $ \|\nabla_{\boldsymbol{w}} L(\boldsymbol{w})\| $ caused by the weight perturbation is larger than that of the feature perturbation. Therefore, in our work, we only apply perturbations to the features.

\subsubsection{Feature Perturbation Quantization}

Feature Perturbation Quantization (FPQ) introduces the perturbations into features, pushing the quantized model to live at a flat local minima. Without loss of generality, taking DNN as an example, we assume that a network has $N$ convolutional layers. Given the input sample $\boldsymbol{x}$, the inputs of each layer for a quantized model are represented as, $\left \{  \boldsymbol{x}_{s}^{1} ,  \boldsymbol{x}_{s}^{2},\dots, \boldsymbol{x}_{s}^{N} \right \} $. We apply perturbations to the inputs of each convolutional layer of quantized model as follows:
\begin{equation}\label{eqt:FPQ}
\boldsymbol{x}^l_s = \boldsymbol{x}^l_s + \boldsymbol{\boldsymbol{\delta}}, \boldsymbol{\delta} \sim  U[-\frac{\boldsymbol{s}^l}{2}, \frac{\boldsymbol{s}^l}{2}]
\end{equation}
where $\boldsymbol{x}^l_s$ is the input feature of the $l$-th layer of quantized model, $\boldsymbol{\delta}$ is the injected perturbation which follows a uniform distribution $U[-\frac{\boldsymbol{s}^l}{2}, \frac{\boldsymbol{s}^l}{2}]$, and $\boldsymbol{s}^l$ is the quantization parameter scale corresponding to the input feature $\boldsymbol{x}^l$. However, simply injecting perturbations by FPQ could disrupt the features of each layers of DNN, empirically resulting in inferior results.


Proposition~\ref{pro:biased_expectation} discovers that if $\boldsymbol{w}^{l}$ represents the weights of the $l$-th convolutional layer  and $\boldsymbol{x}^{l}$ is the corresponding input feature, the expected perturbation $\mathbb{E}[\boldsymbol{\delta}]$ is always biased.

\newtheorem{myobr}{Observation}
\newtheorem{mydef}{Definition}
\newtheorem{mythe}{Theorem}
\newtheorem{mypro}{Proposition}

\begin{mypro}
\label{pro:biased_expectation}
(Inject noise into multiple layers would result in a biased expectation) Without loss of generality, given a DNN with $N (N\ge2)$ layers, if noises $\boldsymbol{\delta}^l (1\leq l \leq N)$  are simultaneously injected into different layers,  the accumulated perturbation from different layers is biased.  
\end{mypro}
\begin{proof}
Without loss of generality, a DNN is represented as 
\begin{equation}
\label{eqt:feed-forward}
f_N(\boldsymbol{x})=\left(\phi_{N} \ldots \circ \phi_{l} \circ \ldots \circ \phi_{0}\right)(\boldsymbol{x}),
\end{equation}
where $\phi_l$ could be any module of the neural network, \textit{ e.g.}, conventional layer, attention layer. 

For the $l$-th layer, the perturbation comes from two sources: the first one is the perturbation $\boldsymbol{\delta}^{l}$ that we directly injected by Eq.~\eqref{eqt:FPQ}; and the second one is from the perturbation propagated from the last layer $\phi_{l-1}$. 
Therefore, the expectation of the perturbation of the $l$-th layer $\mathbb{E}\left ( \boldsymbol{\delta} _{\text{total}} \right )$ is as follows:
\begin{equation}
\begin{small}
\begin{aligned}
\label{eqt:n_layer_perturbation}
    &\quad \,\mathbb{E}\left ( \boldsymbol{\delta} _{\text{total}} \right ) \\
    & =\mathbb{E}\left [ \boldsymbol{\delta}^{l} +  f_{l-1}\left ( \boldsymbol{x}^{l-1}+\boldsymbol{\delta}^{l-1} \right )
    -f_{l-1}\left (  \boldsymbol{x}^{l-1}\right )    \right ]  \\
    &= \mathbb{E}\left ( \boldsymbol{\delta}^{l} \right ) +\mathbb{E}\left [  f_{l-1}\left (\boldsymbol{x}^{l-1}+\boldsymbol{\delta}^{l-1} \right ) -f_{l-1}\left ( \boldsymbol{x}^{l-1}\right )    \right ].
\end{aligned}
\end{small}
\end{equation}
The first term in Eq.~\eqref{eqt:n_layer_perturbation} is a uniform distribution with a expectation of $\mathbb{E}\left ( \boldsymbol{\delta}^{l} \right )=0$. Consequently, we have:
\begin{equation}\label{eqt:perturbation_propagation}
\begin{aligned}
    &\quad \,\mathbb{E}\left ( \boldsymbol{\delta} _{\text{total}} \right ) \\
    &= \mathbb{E}\left [  \left ( f_{l-1}\left ( \boldsymbol{x}^{l-1}+\boldsymbol{\delta}^{l-1} \right ) -f_{l-1}\left ( \boldsymbol{x}^{l-1}\right )  \right )  \right ].
\end{aligned}
\end{equation}

Assuming that the function $f_{l-1}$ is smooth at the point $\boldsymbol{x}^{l-1}$, and that the perturbation $\boldsymbol{\delta}^{l-1}$ is NOT sufficiently small, we apply the second-order Taylor approximation to Eq.~\eqref{eqt:perturbation_propagation} as follows:
\begin{equation}\label{eqt:noise_bias}
\begin{small}
\begin{aligned}
   &\quad \, \mathbb{E}\left(\delta_{\text{total}}\right) \\
   &\approx \mathbb{E}\left[\nabla f_{l-1}(\boldsymbol{x}^{l-1})^T \boldsymbol{\delta}^{l-1} + \frac{1}{2} (\boldsymbol{\delta}^{l-1})^T \nabla^2 f_{l-1}(\boldsymbol{x}^{l-1}) \boldsymbol{\delta}^{l-1}\right] \\
   &= \mathbb{E}\left[\frac{1}{2} (\boldsymbol{\delta}^{l-1})^T \nabla^2 f_{l-1}(\boldsymbol{x}^{l-1}) \boldsymbol{\delta}^{l-1}\right] \ne 0.
\end{aligned}
\end{small}
\end{equation}
\end{proof}

To reduce the bias, we propose a Stochastic Feature Perturbations (SFP) with a predefined probability $p (p>0)$. Concretely, the $l$-th features $\boldsymbol{x}^{l}$ is stochastically injected noise as follows:
\begin{equation}\label{eqt:FPQ_with_p}
\boldsymbol{x}^l_s = \boldsymbol{x}^l_s + \boldsymbol{\delta},\text{with probability }p.
\end{equation}

With SFP, the probability of injecting perturbations into all layers simultaneously is $p^{N}$, which significantly reduces the occurrence of the bias in Eq.~\eqref{eqt:noise_bias}, \textit{e.g.}, $0.1^{10}=1e{-}11$ for a DNN with 10 layers. Besides, we combined with a feature distillation called Channel-wise Standardization Distillation (CSD) to make the Condition 1 hold.

\subsubsection{Channel-wise Standardization Distillation}

\begin{figure}[h!]
    \centering
    \includegraphics[width=0.95\linewidth]{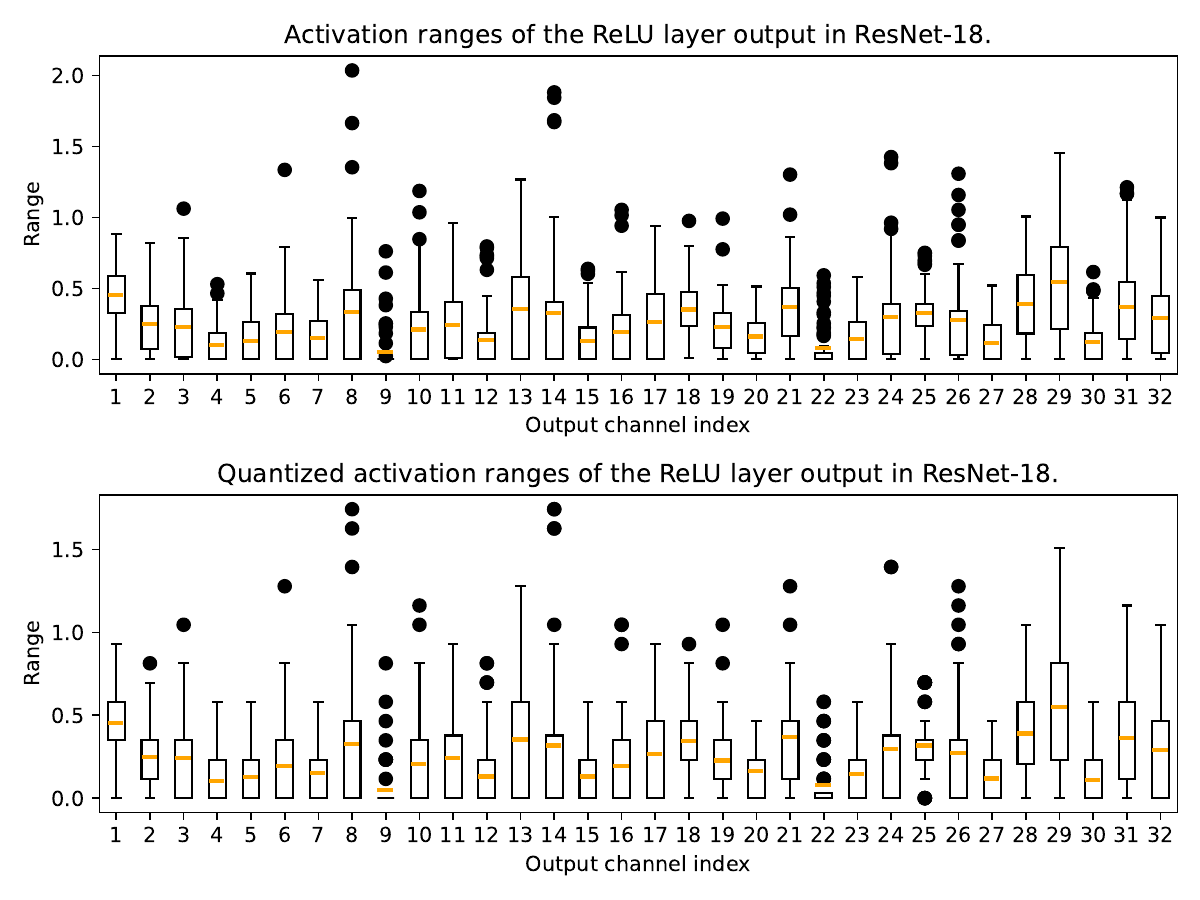}
    \caption{Comparisons of the output feature of the ResNet-18 before and after quantization. The orange line represents the mean value of the feature (best viewed in color).} 
    \label{fig:activation_cmp}
\end{figure}

We have observed that the quantization operation can affect the distribution of features. We visualized the output features of a ReLU layer in the ResNet-18 model to compare the changes in features before and after quantization. The results are shown in Fig.~\ref{fig:activation_cmp}. The findings indicate that the quantization operation can cause a shift in the distribution. Specifically, in Fig.~\ref{fig:activation_cmp}, the feature with index 30 had a mean value of 0.12 before quantization, while after quantization, the mean value shifted to 0.11, resulting in a relative error of nearly 10\%. The drifted distribution is unfriendly for distillation~\cite{nagel-dfq-iccv-2019}~\cite{he-ptqd-nips-2024}. In this work, To mitigate drifted distribution caused by quantization, we standardize features before distillation for forcing Mean Squared Error (MSE) to focus the feature distribution rather than the absolute values as follows: 
\begin{equation}
\label{eqt:standardization}
 \tilde{\boldsymbol{z}}^{l} = Norm(\boldsymbol{z}^{l}) = \frac{\boldsymbol{z}^{l}-\mu \left ( \boldsymbol{z}^{l} \right ) }{\sqrt{\sigma ^{2} \left ( \boldsymbol{z}^{l} \right ) + \epsilon } } ,
\end{equation}
where $\boldsymbol{z}^{l}$ is the output of the $l$-th layer, $\tilde{\boldsymbol{z}}^{l}$ is the normalized feature, $\mu \left ( \boldsymbol{z}^{l} \right )$ is the mean of feature $\boldsymbol{x}^{l}$, $\sigma ^{2} \left ( \boldsymbol{z}^{l} \right ) $ is the variance of the feature, and $\epsilon$ is the constant preventing the denominator being 0.

The student model (quantized) $f_S\left (\boldsymbol{x} ; \boldsymbol{w}_{s}, \boldsymbol{s}_s, \boldsymbol{z}_s\right )$ is initialized with a calibration dataset, while the teacher model (FP)  $f_T\left (\boldsymbol{x} ; \boldsymbol{w}_{t}, \boldsymbol{s}_t, \boldsymbol{z}_t\right )$ serves as a reference. The standardized outputs of the $l$-th layer in the quantized model and that of the FP model are as follows:
\begin{equation}
\label{eqt:output_standardization}
    \begin{aligned}
    \tilde{\boldsymbol{z}}_{s}^{l}=Norm(\hat{\boldsymbol{w}}_{s}^{l} \boldsymbol{x}_{s}^{l}),\\
    \tilde{\boldsymbol{z}}_{t}^{l}=Norm(\boldsymbol{w}_{t}^{l} \boldsymbol{x}_{t}^{l}),
    \end{aligned}
\end{equation}
where $\tilde{\boldsymbol{z}}_{s}^{l}$ and $\tilde{\boldsymbol{z}}_{t}^{l}$ are the standardized features of the $l$-th layer of the quantized model and the FP one, respectively.

To make Condition 1 hold, one reasonable assumption is that if the output of student model  $f_S\left (\boldsymbol{x} ; \boldsymbol{w}_{s}, \boldsymbol{s}_s, \boldsymbol{z}_s\right )$ is equal to that of the teacher one  $f_T\left (\boldsymbol{x} ; \boldsymbol{w}_{t}, \boldsymbol{s}_t, \boldsymbol{z}_t\right )$, the gradient $\nabla_{\hat{\boldsymbol{w}}} L(\hat{\boldsymbol{w}}) $ tends to be zero.  The CSD is as follows:
\begin{equation}
\label{eqt:csd}
L_{CSD} =\sum_{i\in \left [ 1,N \right ] }^{} \sum_{j\in [1,c]}^{} \left \|\tilde{\boldsymbol{z}}_{t}^{i;j}-\tilde{\boldsymbol{z}}_{s}^{i;j} \right \| _{2}^{2} ,
\end{equation}
where $N$ represents the number of layers, $c$ represents the number of channels output by the $i$-th convolution, $\tilde{\boldsymbol{z}}_{s}^{i;j}$, $\tilde{\boldsymbol{z}}_{t}^{i;j}$ represent the standardized features of the $j$-th channel of the $i$-th convolution of the quantized model and the FP model, respectively, $\left \| \cdot  \right \| _{2} $ denotes $\ell _{2} $ normalization.

\subsection{Implicit Regularization on Hessian Matrix}

In the following, we first explain why the spectral norm of Hessian is correlated with solution quality, and then formally show that the difference of our algorithms with Sharpness-aware Minimization (SAM).

\textbf{Why is Hessian norm correlated with solution quality?}

Assume $\boldsymbol{w}^*$ is the optimal weight of Eq.~\eqref{eqt:Loss_example}. Based on Taylor expansion and assume $\nabla L_{\text{val}}(\boldsymbol{w}^*) = 0$ due to optimality condition, we have
\begin{equation}
\label{eqt:hesesian}
L_{\text{val}}(\hat{\boldsymbol{w}}) = L_{\text{val}}(\boldsymbol{w}^*) + \frac{1}{2} (\hat{\boldsymbol{w}} - \boldsymbol{w}^*)^T \hat{H} (\hat{\boldsymbol{w}} - \boldsymbol{w}^*), 
\end{equation}
where $\hat{H} = \int_{\boldsymbol{w}^*}^{\hat{\boldsymbol{w}}} \nabla^2 L_{\text{val}}(\boldsymbol{w}) \, d\boldsymbol{w}$ is the average Hessian. If we assume that Hessian is stable in a local region, then the quantity of $C = \|\nabla^2 L_{\text{val}}(\boldsymbol{w}^*)\| \|\hat{\boldsymbol{w}} - \boldsymbol{w}^*\|$ can approximately bound the performance drop when projecting $\boldsymbol{w}^*$ to $\hat{\boldsymbol{w}}$. After fine tuning, if $\hat{\boldsymbol{w}}$ is the new optimal weight vector, we expect $L_{\text{val}}( \hat{\boldsymbol{w}})$ to be smaller than $L_{\text{val}}( \boldsymbol{w}^{*})$ , provided that the training and validation losses are highly correlated. Therefore, the performance of $L_{\text{val}}(\hat{\boldsymbol{w}})$, which is the quantity that we care for, will also be bounded by $C$. Note that the bound $C$ could be quite loose since that the network weights changed when transitioning from $\boldsymbol{w}^*$ to $\hat{\boldsymbol{w}}$. A more precise bound can be obtained by treating $g(\boldsymbol{w}) = L_{\text{val}}(\boldsymbol{w})$ as a function only parameterized by $\boldsymbol{w}$, and then calculate its derivative/Hessian by implicit function theory.

\textbf{The difference between our method and SAM.}
The concept of searching for minima characterized as ``flat minima'' was introduced in \cite{hochreiter-flatmini-1994-nips} and extensive research has been conducted to explore its connection with the generalization ability~\cite{andriushchenko-understanding-pmlr-2022}~\cite{zhang-1understanding-acm-2021} . SAM~\cite{foret-sam-2020-arxiv} enhances the generalization ability of the model as follows:
\begin{equation}
\label{eqt:sam}
\begin{aligned}
\min_{\boldsymbol{w}} \quad & L^{SAM}(\boldsymbol{w}) + \lambda \|\boldsymbol{w}\|^2, \\
\text{where} \quad & L^{SAM}(\boldsymbol{w}) = \max_{\|\boldsymbol{\epsilon}\| \leq \rho} L(\boldsymbol{w} + \boldsymbol{\epsilon}),
\end{aligned}
\end{equation}
in which $\boldsymbol{\epsilon}$ represents weight perturbations in an Euclidean ball within the radius $\rho$, 
$L^{SAM}$ is the perturbed loss, and $\lambda \|\boldsymbol{w}\|^2$ is the standard L2 norm. 

In contrast, the proposed FPQ is formally defined as follows:
\begin{equation}
\begin{aligned}
\min_{\boldsymbol{w}} & \quad   L_{\boldsymbol{\epsilon} \sim U[-\frac{\boldsymbol{s}^l}{2}, \frac{\boldsymbol{s}^l}{2}]}(\boldsymbol{w} ;\boldsymbol{x}^l+\boldsymbol{\epsilon}),\\
& s.t.\quad  \mathbb{E}[\boldsymbol{\delta}]\cdot \nabla_{\boldsymbol{w}} L(\boldsymbol{w})=0,
\end{aligned}
\end{equation}
where $ \boldsymbol{x}^l$ is the input feature of the $l$-th layer in a DNN. 

\begin{table}[h]
\centering
\caption{The comparison of results for optimizing LSQ using SAM and FPQ, and the accuracy (\%) of the model with W4A4 quantization on the CIFAR-10 dataset.}
\label{tab:cmp_sam_fpq}
\begin{tabular}{ccc}
\toprule
Methods & {Res18} & {MBV2} \\
\midrule
Full prec. & 88.72 & 85.81 \\
LSQ+SAM    & 89.75 & {84.72} \\
FPQ    & \textbf{90.16} & \textbf{85.53} \\
\bottomrule
\end{tabular}
\end{table}

Rather than directly perturbing the weight in Eq.~\eqref{eqt:sam}, we instead perturbing the features through Eq.~\eqref{eqt:FPQ}. 
The comparison of results between FPQ and SAM is shown in Tab.~\ref{tab:cmp_sam_fpq}.  As discussed in Eq.~\eqref{eqt:talor_3}, FPQ imposes an implicit regularization on the Hessian matrix. The results indicate that the optimization using SAM performs worse than that using FPQ. One possible reason is that in high-dimensional spaces, there are many directions of perturbations $\boldsymbol{\delta}$ and their size $ \rho $ to be chosen, and the random approach in FPQ may outperform the method of selecting the largest perturbation in SAM for QAT.


\subsection{Training Processing}
\label{sec:training}

Given a labeled sample $(X,y)$ and the corresponding label $y$, the output of the quantized model and the FP model are:

\begin{equation}
\label{eqt:forward}
    \begin{aligned}
     \boldsymbol{z}_{s}^{o} =f_S\left (\boldsymbol{x} ; \hat{\boldsymbol{w}}_{s}, \boldsymbol{s}_s\right ), \\
     \boldsymbol{z}_{t}^{o} =f_T\left (\boldsymbol{x} ; \hat{\boldsymbol{w}}_{t}, \boldsymbol{s}_t\right ) .
    \end{aligned}
\end{equation}

In the forward through Eq.~\eqref{eqt:forward}, we apply with stochastic perturbation $\boldsymbol{\delta}$ in all the input of convolutional layer with a predefined probability $p$. The perturbation drawn from a uniform distribution through Eq.~\eqref{eqt:FPQ}.

In this paper, we focus on the classification task. Therefore, the whole loss consists of two components as follows:
\begin{equation}
\label{eqt:compositedloss}
L_{total}\left ( \hat{\boldsymbol{w}},\boldsymbol{x} \right ) = CE\left ( \boldsymbol{z}_{s}^{o}, y \right ) +   L_{CSD} ,
\end{equation}
where $CE\left ( \cdot ,\cdot  \right )$ denotes the cross-entropy loss for classification tasks, $z_{o}$ presents the output from the quantized model for the labeled data $(X,y)$. Training procedure is shown in Algorithm \ref{alg:csd}.

\begin{algorithm}[t!]
   \caption{QAT with FPQ}
   \label{alg:csd}
\begin{algorithmic}[1]
   \STATE {\bfseries Input:} 
      labeled data $\boldsymbol{X}$;
       FP model $ f\left (\boldsymbol{x} ; \boldsymbol{w}\right ) $
   \FOR{$i=1$ {\bfseries to} $\cdots$}
        \STATE Feed the  $\boldsymbol{X}$ into models based on Eq.~\eqref{eqt:forward},  while injecting perturbations into all layers of a DNN with probability $p$ by Eq.~\eqref{eqt:FPQ_with_p};
        \STATE Obtain and standarize the outputs of each layer of the quantized model and the FP model based on Eq.~\eqref{eqt:output_standardization};
        \STATE Update for student model based on Eq.~\eqref{eqt:csd} and Eq.~\eqref{eqt:compositedloss};
    \ENDFOR \\
    \STATE {\bfseries Output:} 
      $f_S\left (\boldsymbol{x} ; \boldsymbol{w}_{s}, \boldsymbol{s}_s, \boldsymbol{z}_s\right )$.
\end{algorithmic}
\end{algorithm}

\section{Experiments}


\textbf{Experimental Protocols and Datasets.} Our code is based on PyTorch~\cite{paszke-pytorch-nips-2019} and relies on the MQBench~\cite{li-mqbench-arxiv-2021} package. We used asymmetric quantization by default. In this paper, we used the CIFAR-10 and CIFAR-100~\cite{krizhevsky-cifar-arxiv-2009} as dataset. We randomly selected 100 images for CIFAR-10 and CIFAR-100 as the calibration set. We also kept the first and last layers with 8-bit quantization, the same as QDrop~\cite{wei-qdrop-arxiv-2022}. Additionally, we employed per-channel quantization for weight quantization. We used WXAX to represent X-bit weight and activation quantization. 

\textbf{Training Details.} We used SGD as the optimizer, with a batch size of 256 and a base learning rate of 0.01. The default learning rate (LR) scheduler followed the cosine annealing method. The weight decay was 0.0005, and the SGD momentum was 0.9. We trained for 200 epochs otherwise specified.

\begin{table}[h!]
\begin{center}
\centering
\caption{Ablation studies of perturbation and CSD (Accuracy $\%$) on CIFAR-10 with W2A4 quantization.}
\label{tal：abl_method}
\begin{tabular}{ccc}
\toprule
\begin{tabular}[c]{@{}l@{}}
Perturbations
\end{tabular} &
\begin{tabular}[c]{@{}l@{}}CSD
\end{tabular} &
\begin{tabular}[c]{@{}l@{}}ResNet-18\\
\end{tabular}

\\ \midrule
                  &                               &88.36                                 \\
 \checkmark      &                               &89.30                                 \\
                 & \checkmark                    &89.66                                 \\
 \checkmark      & \checkmark                    &\textbf{89.92}                                 \\
\bottomrule
\end{tabular}
\end{center}
\end{table}

\subsection{Ablation Study}
\label{sec4.1}

\textbf{Effectiveness of probability FPQ.} To investigate the impact of FPQ, we conducted ablation experiments to validate the effects of the perturbation and $L_{CSD}$. We used the W2A4 model of ResNet-18~\cite{he-resnet-cvpr-2016} with the LSQ method as the baseline, which achieved an accuracy of 88.36$\%$ on CIFAR-10 dataset. Tab.~\ref{tal：abl_method} illustrated that FPQ has improved the accuracy of W2A4 quantized models for ResNet architectures. Adding only the perturbation in Eq.~\eqref{eqt:FPQ_with_p} improved the baseline by 0.94\%, using only CSD improved the baseline by 1.3\%, and when both were used together, the result was better, improving the baseline by 1.56\%.

\begin{table}[htbp]
\centering
\caption{The effectiveness of the probability $p$ on CIFAR-10  with W2A4 quantization (Accuracy $\%$).}
\label{tab:abl_p}
\begin{tabular}{cc}
\toprule
{$p$ of FPQ in Eq.~\eqref{eqt:FPQ_with_p}} & {Val acc (\%)} \\
\midrule
{0} & 89.66 \\ 
{0.1} & \textbf{89.92} \\
{0.3} & 89.90 \\
{0.5} & 89.83 \\
{0.7} & 89.59 \\
{0.9} & 89.68 \\
{1} & 89.48 \\
\bottomrule
\end{tabular}
\end{table}

\textbf{Effectiveness of probability $p$.} To investigate the impact of probability $p$ in Eq.~\eqref{eqt:FPQ_with_p}, we used CIFAR-10 as the dataset. Tab.~\ref{tab:abl_p} illustrated that the performances of the W2A4 model with the seven different sets of $p$. The experimental results indicated that when the random probability $p$ is too high, excessive perturbations can disrupt the local minima of the quantized model. Conversely, when the perturbations are too small, the insufficient noise failed to smooth the sharpness of the model's loss landscape. An optimal probability $p$ lies within the range of 0.1 to 0.5.

\begin{figure}[t!]
    \centering
    \includegraphics[width=0.95\linewidth]{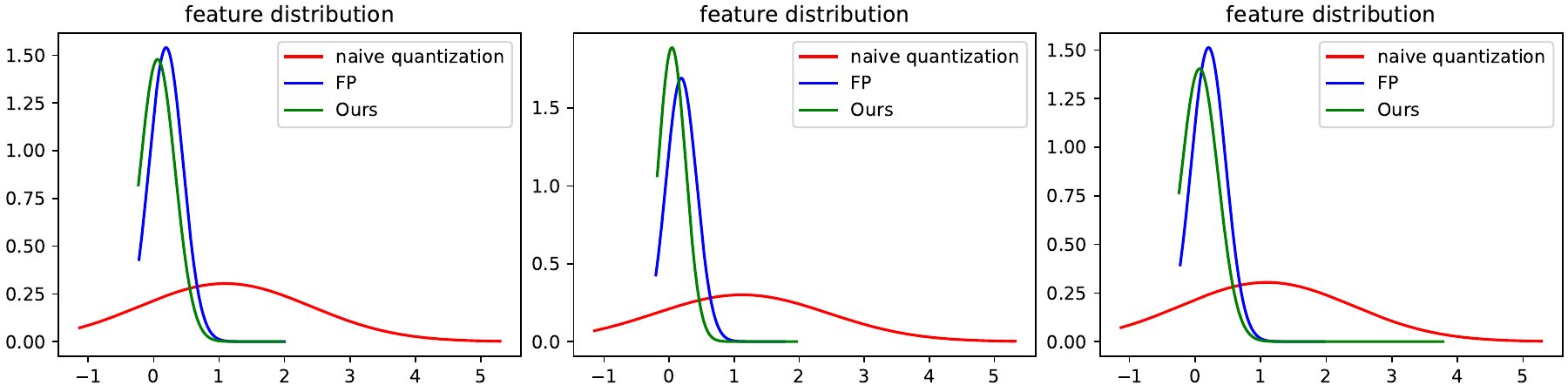}
    \caption{Feature distributions of the FP ResNet-18 (blue lines), native baseline (LSQ, red lines), and ours method (green lines) for the same feature of ResNet-18 model on CIFAR-10. We random select three samples and plot the output feature of the same layer.} 
    \label{fig:feature_distribution}
\end{figure}

\subsection{Literature Comparison}
\label{sec4.2}

 \begin{table}[h!]
\begin{center}
\centering
\caption{Comparison among different QAT strategies in terms of accuracy on CIFAR-10.}
\label{tbl:comp-CIFAR-10}
\resizebox{0.95\linewidth}{!}{
\begin{tabular}{clrcccc}
\toprule
\multicolumn{1}{c}{\begin{tabular}[c]{@{}c@{}}\textbf{Labeled} \\ \textbf{data}\end{tabular}} & \multicolumn{1}{l}{\textbf{Methods}} & \multicolumn{1}{c}{\textbf{W/A}} & \multicolumn{1}{c}{\textbf{Res18}} & \multicolumn{1}{c}{\textbf{Res50}}  & \multicolumn{1}{c}{\textbf{MBV1}} & \multicolumn{1}{c}{\textbf{MBV2}} \\ \midrule
50000                       & Full Prec.                           & 32/32 & 88.72         &89.95                     &85.52           &85.81      \\ \hline\hline
\multirow{15}{*}{50000}     & PACT~\cite{choi-pact-arxiv02018}     & 4/4   &88.15          &85.27                     &80.77           &79.88      \\
                            & LSQ~\cite{esser-lsq-arxiv-2019}      & 4/4   &86.69          &90.01                     &82.39           &84.45      \\
                            & LSQ+~\cite{bhalgat-lsq+-cvpr-2020}   & 4/4   &88.40          &90.30                     &84.32           &84.30  \\
                            & KD~\cite{hinton-kd-arxiv-2015}       & 4/4   &88.86          &90.34                     &84.77           &83.79  \\
                            & FPQ (Ours)                           & 4/4   &\textbf{90.16} &\textbf{90.62}            &\textbf{84.94}  &\textbf{85.53}     \\ \cmidrule{2-7}
                            & PACT~\cite{choi-pact-arxiv02018}     & 2/4   &87.55          &85.24                      &69.04           &67.18      \\
                            & LSQ~\cite{esser-lsq-arxiv-2019}      & 2/4   &88.36          &90.01                      &78.15           &78.15      \\
                            & LSQ+~\cite{bhalgat-lsq+-cvpr-2020}   & 2/4   &87.76          &89.62                      &81.26           &77.00      \\
                            & KD~\cite{hinton-kd-arxiv-2015}       & 2/4   &88.83          &90.18                      &78.84           &75.56  \\
                            & FPQ (Ours)                           & 2/4   &\textbf{89.92} &\textbf{90.39}             &\textbf{81.45}           &\textbf{79.33}      \\ \cmidrule{2-7}
                            & PACT~\cite{choi-pact-arxiv02018}     & 2/2   &76.90          &64.94              &11.71           &10.58        \\
                            & LSQ~\cite{esser-lsq-arxiv-2019}      & 2/2   &87.60          &87.79              &75.29           &70.32        \\
                            & LSQ+~\cite{bhalgat-lsq+-cvpr-2020}   & 2/2   &87.60          &86.10              &74.22           &72.18         \\
                            & KD~\cite{hinton-kd-arxiv-2015}       & 2/2   &88.06          &89.20              &68.62           &67.15         \\
                            & FPQ (Ours)                           & 2/2   &\textbf{88.11} &\textbf{89.40}      &\textbf{76.25}  &\textbf{73.00}      \\\bottomrule

\end{tabular}
}
\end{center}
\end{table}

We selected ResNet-18 and ResNet-50~\cite{he-resnet-cvpr-2016}, MobileNetV1~\cite{howard-mbv1-arxiv-2017} and MobileNetV2~\cite{sandler-mbv2-cvpr-2018} with depth-wise separable convolutions as the representative network architectures. 

\textbf{CIFAR-10.} In Tab.~\ref{tbl:comp-CIFAR-10}, we quantized the weights and activations to 2-bit and 4-bit. We compared our approach with the effective baselines, including LSQ~\cite{esser-lsq-arxiv-2019}, LSQ+~\cite{bhalgat-lsq+-cvpr-2020}, PACT~\cite{choi-pact-arxiv02018} and KD~\cite{hinton-kd-arxiv-2015}. Tab.~\ref{tbl:comp-CIFAR-10} illustrated that when the entire training set of CIFAR-10 is used, FPQ significantly surpassed the baselines. In W4A4 quantization, FPQ achieved about 1$\sim$2$\%$ accuracy improvements over LSQ. Furthermore, to explore the ability of FPQ, we conducted W2A4 and W2A4 quantization experiments. In W2A4 quantization, FPQ consistently achieved a 1$\sim$2$\%$ accuracy improvement over LSQ in Tab.~\ref{tbl:comp-CIFAR-10}. In W2A2 setting, FPQ achieved about 1$\sim$3$\%$ accuracy improvements over LSQ. Moreover, there are two interesting observations as follows:
\begin{itemize}
    \item For W4A4, our method significantly surpassed the FP counterparts for both ResNet-18 and ResNet-50.  For example, on the ResNet-18 model, FPQ surpassed the FP model by 1.44\% in accuracy, and on the ResNet-50 model, FPQ exceeded the FP model by 0.67\%.
    \item From W4A4 to W2A2, the performance drop of our method is significantly lower that the other SOTA methods. For instance, on the ResNet-50 model, the LSQ method decreased by 2.22\% when reducing from W4A4 to W2A2, while FPQ decreased by 1.22\%. On the MobileNetV1 model, the LSQ+ method saw a 10.1\% drop when going from W4A4 to W2A2, while FPQ decreased by 8.69\%
\end{itemize}

\begin{table}[t!]
\begin{center}
\centering
\caption{Comparison among different QAT strategies regarding accuracy on CIFAR-100.}
\label{tbl:comp-CIFAR-100}
\resizebox{0.95\linewidth}{!}{
\begin{tabular}{clrcccc}
\toprule
\multicolumn{1}{c}{\begin{tabular}[c]{@{}c@{}}\textbf{Labeled} \\ \textbf{data}\end{tabular}} & \multicolumn{1}{l}{\textbf{Methods}} & \multicolumn{1}{c}{\textbf{W/A}} & \multicolumn{1}{c}{\textbf{Res18}} & \multicolumn{1}{c}{\textbf{Res50}}  & \multicolumn{1}{c}{\textbf{MBV1}} & \multicolumn{1}{c}{\textbf{MBV2}} \\ \midrule
50000                       & Full Prec.                                   & 32/32 &75.40           & 78.94                   &70.22           &71.30      \\ \hline\hline
\multirow{15}{*}{50000}     & PACT~\cite{choi-pact-arxiv02018}             & 4/4   &74.17           &74.78                    &64.65           &64.06      \\
                            & LSQ~\cite{esser-lsq-arxiv-2019}              & 4/4   &75.30           &78.20                    &68.63           &69.01      \\
                            & LSQ+~\cite{bhalgat-lsq+-cvpr-2020}           & 4/4   &74.50           &77.39                    &67.89           &68.25  \\
                            & KD~\cite{hinton-kd-arxiv-2015}               & 4/4   &74.70           &78.80                    &\textbf{70.96}  &\textbf{71.66} \\
                            & FPQ (Ours)                                   & 4/4   &\textbf{75.84}  &\textbf{78.8}            &67.55           &68.24     \\ \cmidrule{2-7}
                            & PACT~\cite{choi-pact-arxiv02018}             & 2/4   &73.77          &74.72                     &49.98           &57.90      \\
                            & LSQ~\cite{esser-lsq-arxiv-2019}              & 2/4   &74.93          &77.82                     &65.13           &66.15      \\
                            & LSQ+~\cite{bhalgat-lsq+-cvpr-2020}           & 2/4   &73.90          &76.61                     &65.28           &\textbf{66.24}      \\
                             & KD~\cite{hinton-kd-arxiv-2015}              & 2/4   &74.35          &77.34                     &\textbf{66.90}  &63.77  \\
                            & FPQ (Ours)                                   & 2/4   &\textbf{75.77} &\textbf{78.11}            &64.68           &65.35      \\ \cmidrule{2-7}
                            & PACT~\cite{choi-pact-arxiv02018}             & 2/2   &65.16          &4.26                      &3.25            &8.39        \\
                            & LSQ~\cite{esser-lsq-arxiv-2019}              & 2/2   &71.80          &62.79                     &55.53           &31.08        \\
                            & LSQ+~\cite{bhalgat-lsq+-cvpr-2020}           & 2/2   &71.25          &64.21                     &55.56           &30.08         \\
                            & KD~\cite{hinton-kd-arxiv-2015}               & 2/2   &\textbf{73.13} &63.16                     &55.37           &28.56         \\
                            & FPQ (Ours)                                   & 2/2   &71.73          &\textbf{65.23}            &\textbf{58.46}  &\textbf{33.16}      \\\bottomrule

\end{tabular}
}
\end{center}
\end{table}

\textbf{CIFAR-100.} In Tab.~\ref{tbl:comp-CIFAR-100}, we evaluate quantization performance across varying bit-widths (W4A4, W2A4, W2A2) on CIFAR-100. Same to CIFAR-10, FPQ demonstrates superior performance in most scenarios. For W4A4 quantization, FPQ achieves the highest accuracy on ResNet-18 (75.84\%) and ResNet-50 (78.8\%), surpassing LSQ by 0.54\% and 0.6\% respectively. Notably, FPQ even exceeds the full-precision baseline by 0.44\% on ResNet-18. In W2A4 settings, FPQ maintains strong performance with ResNet-18 (75.77\%) and ResNet-50 (78.11\%), outperforming LSQ by 0.84\% and 0.29\%. For the challenging W2A2 configuration, FPQ achieves the best results on ResNet-50 (65.23\%), MobileNetV1 (58.46\%), and MobileNetV2 (33.16\%), demonstrating robustness under extreme quantization.

\subsection{Characteristics of FPQ}
\label{sec:Analysis}

\textbf{Generalization of FPQ:} It has been empirically pointed out that the dominant eigenvalue of $\nabla^2 L_{\text{val}}(\boldsymbol{w})$ (spectral norm of Hessian) is highly correlated with the generalization quality of QAT solutions~\cite{keskar-eigenvalue1-iclr-2017} \cite{Wen-eigenvalue2-icais-2020}. In standard QAT training, the Hessian norm is usually great, which leads to deteriorating (test) performance of the solutions. In Fig.~\ref{fig:trace_cpm_r18} and Fig.~\ref{fig:trace_cpm_mbv2}, we plot the Hessian trace during the training procedure and find that the Hessian trace of the proposed methods is significantly lower than that of the LSQ. The results demonstrate that FPQ would reduce the trace of the Hessian matrix numerically, thereby enhancing the model's generalization ability.

\textbf{Distribution Visualization of FPQ:}
We visualized the feature distribution before and after quantization and the FP model to explore whether FPQ could align the distribution of the quantized model with that of the FP model.

As shown in Fig.~\ref{fig:feature_distribution}, there was a significant distribution drift between the naive quantization and the FP model. This problem is the key reason for the drop in performance of the quantized model. Our method effectively alleviated this problem. On the one hand, FPQ leveraged the feature distribution of the FP model as the ground truth, reducing the distribution drift caused by quantization. On the other hand, FPQ retained the original task loss, making stochastic perturbation in Eq.~\eqref{eqt:FPQ_with_p} to smooth the local minima.


    

\begin{figure}[t]
    \centering
    \subfigure[ ResNet-18 ]{
        \includegraphics[width=0.45\linewidth]{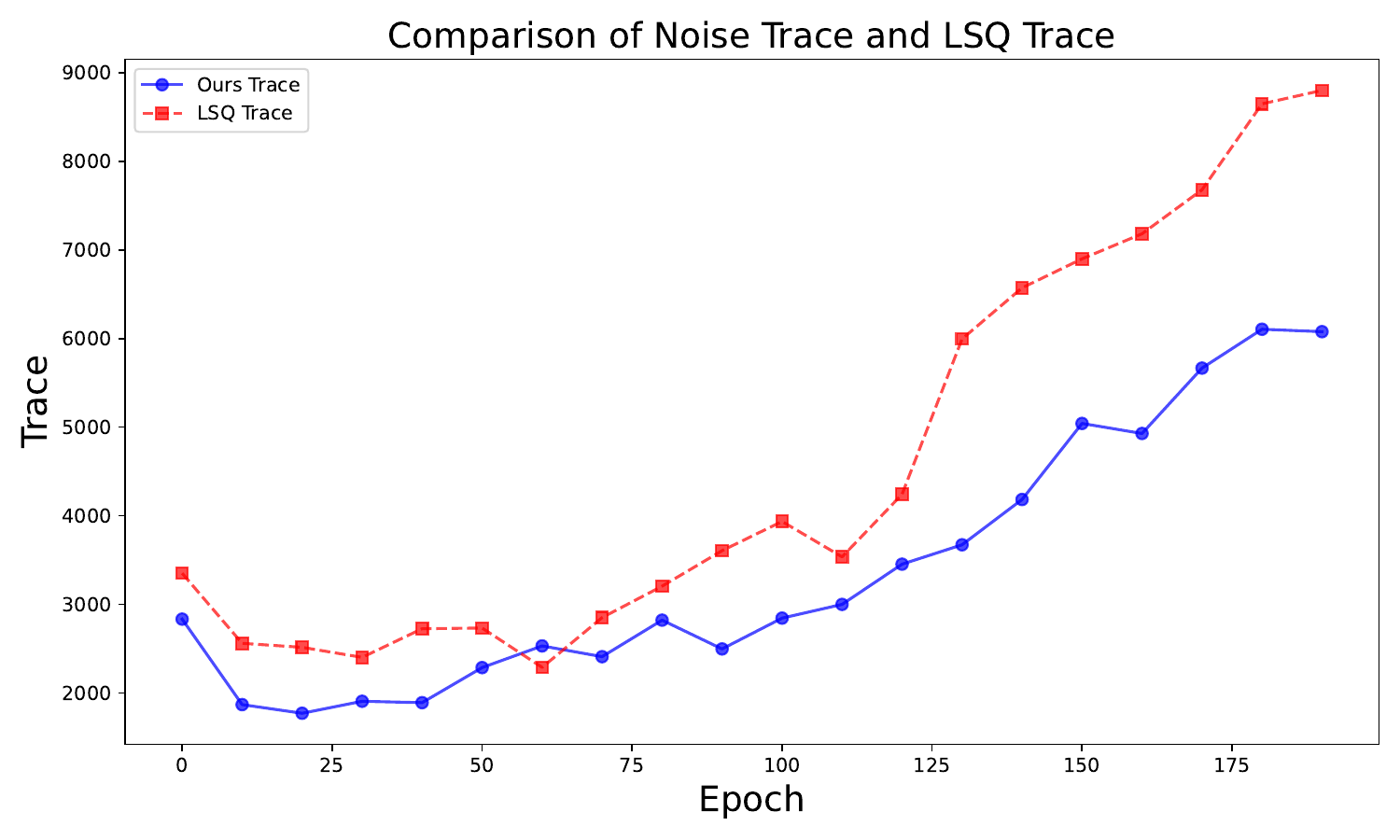}
        \label{fig:trace_cpm_r18}
    }
    \subfigure[MobileNetV2 ]{
        \includegraphics[width=0.45\linewidth]{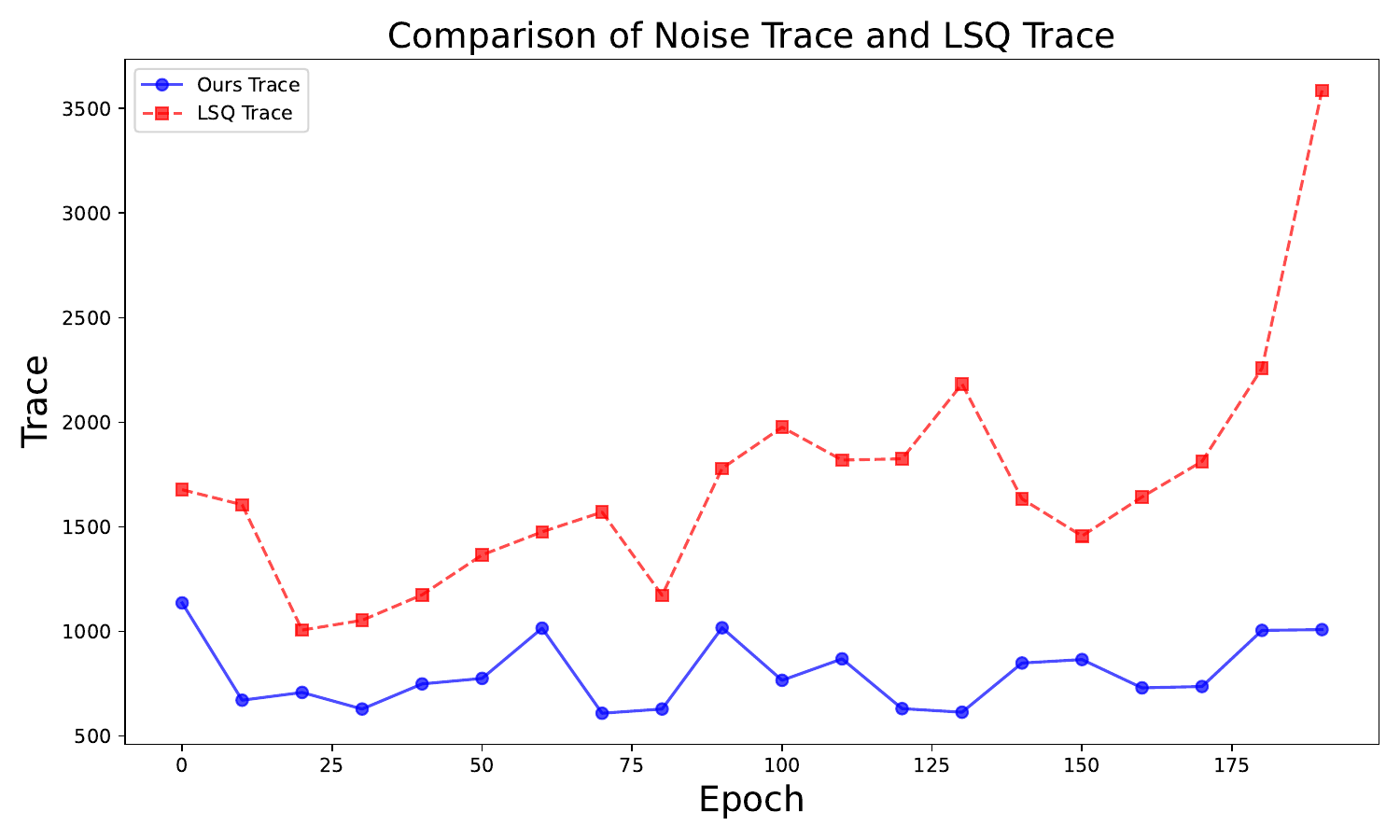}
        \label{fig:trace_cpm_mbv2}
    }
    \caption{Trajectory of the Hessian matrix trace for ResNet-18 and MobileNetV2 models on the CIFAR-10 dataset.}
\end{figure}

\section{Conclusion}
In this paper, we propose the FPQ for QAT-based method. By applying perturbations to the features, FPQ introduces the implicit regularization to the Hessian matrix, enhancing the stability of the model. Specifically, the regularization is carried out with random smoothing. FPQ possesses a much smoother landscape and has the theoretical guarantee to regularize the Hessian norm of the validation loss. Extensive experiments have illustrate the effectiveness of FPQ and we outperform various SOTA methods.

\bibliographystyle{named}
\bibliography{ijcai25}

\end{document}